\documentclass[a4paper,11pt]{article}
\usepackage{amsmath}
\usepackage{amssymb}
\usepackage{amsthm}
\usepackage{algorithm}
\usepackage{algorithmic}
\usepackage{epsf}
\usepackage{epsfig}
\usepackage{graphicx,color,rotating}
\usepackage{multirow,longtable}
\usepackage{graphicx}
\usepackage{ucs}
\usepackage[applemac]{inputenc}
\usepackage{makeidx}  

\newtheorem{thm}{Theorem}[section]

\newtheorem{lem}[thm]{Lemma}
\newtheorem{obs}[thm]{Observation}

\theoremstyle{remark}

\theoremstyle{definition}

\usepackage{makeidx}  
\usepackage{authblk}
\usepackage{mathrsfs}

\begin{document}

\title {Tree Optimization Based Heuristics and Metaheuristics in Network Construction Problems}

\author{ Igor Averbakh$^1$\vspace*{0.25 cm}, Jordi Pereira$^2$ \\
$^1$ Department of Management, University of Toronto Scarborough\\
 and Rotman School of Management, University of Toronto,\\
1265 Military Trail, Toronto, Ontario M1C 1A4, Canada;\vspace*{0.25 cm}\\
$^2$ Faculty of Engineering and Sciences, Universidad Adolfo Ib\'{a}\~{n}ez,\\ Av. Padre Hurtado 750, Vi\~{n}a del Mar, Chile.\vspace*{0.25 cm}\\
emails: averbakh@utsc.utoronto.ca, jorge.pereira@uai.cl}

\maketitle





\textbf{Abstract.} We consider a recently introduced class of network construction problems  where edges of a transportation network need to be constructed by a server (construction crew). The server has a constant construction speed which is much lower than its travel speed, so relocation times are negligible with respect to construction times. It is required to find a construction schedule that minimizes a non-decreasing function of the times when various connections of interest become operational. Most problems of this class are strongly NP-hard on general networks, but are often tree-efficient, that is, polynomially solvable on trees. We develop a generic local search heuristic approach and two metaheuristics (Iterated Local Search and Tabu Search) for solving tree-efficient network construction problems on general networks, and explore them computationally. Results of computational experiments indicate that the methods have excellent performance.


\newpage

\section{Introduction}

In the last decade, a new research area has emerged that deals with operational planning of construction activities in transportation and communication networks. If classical network design problems address the question of \emph{which} network should be constructed to accomplish certain long-term strategic goals, the problems in this area address the question of \emph{how} the network should be constructed with the focus on the network's performance \emph{during} the construction period. In particular, given the available limited resources, the question is in what order different parts of the planned network should be constructed to optimize some short-term or medium-term goals. These issues are important since network construction typically takes long time, resources are always limited, different parts/connections often have different urgency/priority/demand, and operational performance of some connections may be needed much sooner than the whole network can become fully operational. 
This gives rise to optimization models that combine network design and scheduling features. Such models also combine different types of combinatorial structures and technical challenges characteristic for network design and scheduling problems, so they are not amenable to the methods developed in each of these two well-studied fields. 

The literature on network design with scheduling considerations can be classified into two categories. The first category is concerned with \emph{network improvement}: given an already functional and connected network, how should the network be dynamically modified by installing new edges to improve some overall characteristic of network performance. Two closely related lines of research represent this category. Baxter et al. \cite{baxter}, Engel et al. \cite{engel}, and Kalinowski et al. \cite{kalinowski} introduce and study \emph{Incremental Network Design} problems, where at each period of a finite horizon one new edge is installed, and the goal is to optimize the cumulative value over all periods of some performance metric such as the minimum spanning tree length \cite{engel}, shortest path length \cite{baxter}, or maximum flow value \cite{kalinowski}. Cavdaroglu et al. \cite{cavdaroglu}, Nurre et al. \cite{nurre} and Nurre and Sharkey \cite{nurre2} introduce and develop a more general \emph{Integrated Network Design and Scheduling} framework that can model many environments and objectives. In this framework, arcs are installed by a number of identical crews that operate in parallel, and the goal is again to optimize the cumulative value of some network performance metric over a finite horizon. Additional references can be found in \cite{nurre2}.

The second category is specifically concerned with \emph{connectivity issues}, that is, with the process of building the network where the times of establishing / restoring connections between important areas are optimized. An early related work is Guha et al. \cite{guha} who study the power outage recovery problem where connectivity between generator nodes and customer nodes in an electric power network is destroyed as a result of the failure of some relay nodes, and it is required to restore the connectivity by a multi-stage repair of the failed relay nodes. At every stage, only a subset of the failed nodes can be repaired due to a limited budget, and the total waiting time of the disconnected customers is minimized.  In \cite{averbakhIIE, averbakhEJOR, averbakhJOC}, network construction / connectivity restoration problems (further referred to as \emph{Network Construction} (NC) problems) are considered, and these problems will be the focus of this paper. In single-server NC problems, the network is constructed by a server (construction crew) whose construction speed is much slower than the travel speed, so its relocation times within the network are negligible with respect to construction times. The goal is to minimize a scheduling objective related with the times when specific vertices become connected to the operational part of the network, or with the times when certain pairs of vertices become connected. Such problems arise, for example, in planning emergency restoration operations for networks that are partially destroyed/damaged as a result of a disaster, where some vertices / parts of the network have been disconnected from the main network and need urgent re-connection as people/equipment may be trapped there, with different degrees of urgency/priority. Discussions of different types of applications for NC problems can be found in \cite{averbakhIIE, averbakhEJOR, averbakhJOC}. We note that in the context of some applications of NC problems such as emergency restoration operations, even a small difference in solution quality can have a very serious impact, as people's lives and safety may be at stake.

Several NC problems have been studied in  \cite{averbakhIIE, averbakhEJOR, averbakhJOC} with the purpose of developing exact branch-and-bound algorithms based on lower bounds on the optimal objective value derived from an analysis of the combinatorial structure of the problems. Some fast heuristics were also suggested for these problems \cite{averbakhIIE, averbakhEJOR, averbakhJOC}; these heuristics were used for obtaining starting solutions for the branch-and-bound algorithms, and for obtaining approximate solutions for larger instances that could not be solved by exact methods.

Most studied NC problems are strongly NP-hard on general networks \cite{averbakhIIE, averbakhEJOR, averbakhJOC}; on trees, some of them are polynomially solvable and some are NP-hard \cite{averbakhIIE, averbakhEJOR, averbakhJOC}. NC problems that are polynomially solvable on trees are called \emph{tree-efficient}. In this paper, we present a heuristic and two metaheuristics for solving tree-efficient NC problems on general networks. The methods are \emph{generic} in the sense that they can be used for any NC problem that has the tree-efficiency property, and show excellent performance for the specific NC problems studied in \cite{averbakhIIE, averbakhEJOR, averbakhJOC}. In particular, our heuristic outperforms the heuristics from  \cite{averbakhIIE, averbakhEJOR, averbakhJOC}, and the metaheuristics find optimal solutions for almost all instances for which optimal solutions are known.

The paper is organized as follows. In Section \ref{ncp}, we formally define NC problems and provide their classification. Section \ref{localsearch} develops a generic local search heuristic, defines neighborhoods of interest, and provides some auxiliary results. Section \ref{metaheuristics} describes the two metaheuristics that we explore. Section \ref{experiments} discusses computational results, and Section \ref{conclusions} provides some conclusions.

\section{Network construction problems} \label{ncp}

Let us define NC problems that we consider.
Given is a connected undirected network $G=(V,E)$ with a set $V$ of $n$ vertices and a set $E$ of $m$ edges with positive lengths. 
A pair of different vertices of $G$ will be called a \emph{vertex pair}. For any vertex pair $\{v,u\}$, the shortest path distance between $v$ and $u$ will be denoted $d(v,u)$. For any subnetwork (or the set of edges) $G'$, the total length of the edges of $G'$ will be denoted $L(G')$. The edges of the network represent connections (e.g., roads) that need to be constructed; initially, there are only vertices. Construction is done by a server (construction crew) that can construct one unit of length of the network per unit of time, so the length of an edge represents the time required for constructing the edge. The server starts working at time 0. For any vertex pair $\{u, v\}$, its \emph{connection time} $C(u,v)$ is the time when the pair becomes connected by an already constructed path. It is required to find a feasible construction schedule that minimizes some nondecreasing function of connection times of different vertex pairs.

The first key assumption is that travel times are negligible with respect to construction times. That is, if at some time the server wants to relocate from its current position in the network to another point of the network, this relocation, if feasible, can be done instantaneously. The second key assumption is about which relocations are feasible. We consider two settings. In the \textbf{\emph{internal transportation}} (IT) setting, we assume that the server can travel only within the already constructed part of the network and is initially located at a special vertex $r$ called \emph{the depot}. That is, in this setting, only relocations (travel) within the already constructed part of the network (which then will always be connected) are feasible. In the \textbf{\emph{external transportation}} (ET) setting, we assume that at any time, a relocation from any point of the network to any other point is feasible even if these points are not connected currently, since some modes of transportation that do not depend on the network under construction are available to the server. In the external transportation setting, the initial location of the server  is irrelevant for the problem.

Since the objective function is nondecreasing in connection times, it is straightforward to see that there are optimal schedules without preemption (that is, if the server starts constructing an edge, it finishes constructing the edge without interruptions) such that initially the edges of some spanning tree are constructed \cite{averbakhIIE, averbakhEJOR, averbakhJOC}. Hence, we will consider only such schedules. The edges of this spanning tree are called \emph{essential}; all other edges are called \emph{unessential}. Clearly, the choice and the order of construction of essential edges define uniquely the objective function value, since when all essential edges have been constructed, all vertex pairs are connected. Therefore, \emph{a feasible solution is defined by a choice of the spanning tree of essential edges, and a choice of the order of construction for the essential edges}. Any spanning tree is a possible choice for the tree of essential edges. Given a spanning tree of essential edges, any order of their construction is feasible in the external transportation setting. In the internal transportation setting, only the orders of construction that ensure that at any time the already constructed edges form a connected subtree that includes the depot are feasible. In the remainder of the paper, \emph{by a solution we understand a feasible sequence of $n-1$ edges that form a spanning tree}. Such a sequence will be called \emph{e-sequence}. For a solution $S$, let $F(S)$ be its objective function value.

With respect to which connection times are important and can be used to form an objective function, we distinguish between two classes of NC problems. In \emph{pairwise connection} (PC) problems, connection times of any vertex pairs (up to $0.5n(n-1)$ quantities) may be relevant and included in the objective function. In \emph{vertex recovery} (VR) problems, only the connection times of vertex pairs that include the depot $r$   are relevant and may be used in the objective function ($n-1$ quantities); for a vertex $v$,  the corresponding connection time $R(v)=C(v,r)$ is called the \emph{recovery time of $v$}. So, a VR problem minimizes a nondecreasing function of the recovery times of vertices. (In the context of emergency restoration operations, in VR problems, the depot $r$ may represent the main network, and other vertices may represent parts of the original network that have been disconnected from the main network as a result of a disaster, see   \cite{averbakhIIE, averbakhEJOR}.)    Hence, with respect to the setting and relevant connections, we have four groups of NC problems: ITPC, ETPC, ITVR, and ETVR (e.g., ETPC corresponds to external transportation pairwise connection). However,  since the objective function is assumed to be non-decreasing in connection times, the group ETVR has no independent importance and is equivalent to ITVR (it is straightforward to observe that for any ETVR problem, there is always an optimal solution where at any time the already constructed edges form a connected subtree that includes the depot). So, we do not consider ETVR problems.

The following heuristic approaches were used in \cite{averbakhIIE, averbakhEJOR, averbakhJOC} to obtain good approximate solutions on general networks for the tree-efficient NC problems studied in these papers:

\textbf{Heuristic MST}. Find a minimum spanning tree for the network $G$, and solve the problem restricted to this tree (that is, the minimum spanning tree is considered as the tree of essential edges, and the optimal order of constructing these edges is found using the polynomial algorithm for solving the problem on a tree).

\textbf{Heuristic MST combined with some local search}. A local search is used for improving the solution obtained by Heuristic MST. Since NC problems combine scheduling and network design structures, some of the local search approaches used scheduling-based neighborhoods \cite{averbakhIIE, averbakhEJOR} and some used network-based neighborhoods \cite{averbakhJOC}. They extensively used tree-efficiency, by finding the optimal order of constructing the edges for various spanning trees that the algorithms produced as candidates for the tree of essential edges.

It was observed in  \cite{averbakhIIE, averbakhEJOR, averbakhJOC} that for tree-efficient NC problems, Heuristic MST combined with local search typically produced good-quality approximate solutions. This observation suggests that local-search based metaheuristics may be effective for obtaining optimal or near-optimal solutions on general networks for tree-efficient NC problems. (Available exact methods can handle only instances of small and medium sizes.) This idea motivated the research presented in this paper, with the goal of developing local-search based metaheuristics that use tree-efficiency, comparing different designs and types of neighborhoods, and studying their effectiveness. An additional goal was to develop a universal generic local-search based heuristic approach applicable to any tree-efficient NC problems.

This paper makes the following main contributions:

1. We suggest a universal heuristic approach for tree-efficient NC problems on general networks. The approach is a generalized and somewhat enhanced version of the local search heuristic used in \cite{averbakhJOC} for a specific NC problem.  The approach uses network-based neighborhoods, is applicable to different types of tree-efficient NC problems with different objective functions, and outperforms the heuristics from \cite{averbakhIIE, averbakhEJOR, averbakhJOC}.

2. We suggest and study experimentally two local search based metaheuristics (Tabu Search (TS) and Iterated Local Search (ILS)), which are the first metaheuristics developed for NC problems. A comparison of the solutions obtained by the metaheuristics with the exact solutions obtained by the branch-and-bound algorithms in \cite{averbakhIIE, averbakhEJOR, averbakhJOC} shows that the metaheuristics are extremely effective for tree-efficient NC problems, with ILS performing slightly better than TS. The metaheuristics that use network-based neighborhoods found optimal solutions in almost all cases where the optimality could be verified by an exact algorithm; these metaheuristics are applicable to different types of tree-efficient NC problems with different objective functions.

3. By conducting extensive computational experiments for tree-efficient NC problems, we show that network-based neighborhoods are more effective than scheduling-based neighborhoods for local search in the context of both heuristics and metaheuristics, consistently across different classes of tree-efficient NC problems. In particular, our local search heuristic that uses network-based neighborhoods performs significantly better for problems studied in \cite{averbakhIIE, averbakhEJOR} than the local search heuristics from these papers that use scheduling-based neighborhoods; the metaheuristics that use network-based neighborhoods perform significantly better than their versions with scheduling-based neighborhoods.

Some additional conclusions and insights will be mentioned in our detailed discussion of the experimental results.

\section{Local search} \label{localsearch}

\subsection{Neighborhoods and auxiliary procedures}\label{neighborhoods}

Consider a generic NC problem that we will call \textbf{Problem NC}. The only additional assumption we make is that Problem NC is tree-efficient. Then, for any spanning tree $T$, the optimal sequence of constructing the edges of $T$ can be obtained in polynomial time if $T$ is considered as the tree of essential edges. This e-sequence will be denoted $ES(T)$ (ES stands for ``edge sequence") and will be assumed to be unique for $T$; some consistent rule can be used for breaking ties. 
Thus, to specify an optimal solution, it is sufficient to specify its spanning tree of essential edges. 

An \emph{edge-exchange move} for a spanning tree $T$ consists of adding one edge to $T$ and removing one of the edges of the obtained cycle. An edge-exchange move transforms $T$ into another spanning tree; let $N(T)$ be the set of all spanning trees that can be obtained from $T$ by a single edge-exchange move. Observe that $|N(T)|=O(mn)$. For a solution with a spanning tree $T$ of essential edges, its \emph{edge-exchange neighborhood}  is $\{ES(T')~|~T'\in N(T)\}$. That is, to obtain an edge-exchange neighbor of the current solution, we apply an edge-exchange move to its spanning tree of essential edges which results in another spanning tree $T'$, and then obtain the optimal order of constructing the edges of $T'$ using tree-efficiency of Problem NC.  The edge-exchange neighborhood is network-based since a neighbor is defined via a change in the network representation of a solution (the tree of essential edges). In the remainder of the paper, the terms ``edge-exchange neighborhood" and ``network-based neighborhood" will be used interchangeably.

To define scheduling-based neighborhoods, let us first consider IT problems. A sequence of all $n-1$ non-depot vertices of $G$ will be called a \emph{v-sequence}. 
For any v-sequence $S'$, consider the spanning tree that is obtained using the following procedure.

\textbf{Procedure A-IT}. \emph{Input:}  A v-sequence $S'$.\\
\emph{Output:} spanning tree $T'(S')$.\\
Set  $E'=\emptyset$, $V'=\{r\}$, $T'=(V',E')$. (Throughout the course of the procedure, $T'$ will be a subtree of $G$, becoming a spanning tree at the end.)\\
Until  $V'$ includes all vertices of $G$, do\\ 
(Main iteration)\\
\emph{Begin}\\
  $~~~~~$Select the first  vertex in $S'$ that is still not spanned by the current tree $T'$ (that is, not included in $V'$); let it be vertex $v$. Find a shortest path from  $v$ to $T'$ in $G$, and include the edges of this path in $E'$ and the vertices of the path in $V'$. Update $T'=(V', E')$.
 \\
\emph{End};\\
Output $T'(S')=T'$.

Procedure A-IT obtains a spanning tree $T'(S')$ (clearly, no cycles are created). Let us discuss the complexity of Procedure A-IT. We assume that as a \textbf{pre-processing}, the Floyd-Warshall algorithm \cite{ahuja} is run that finds the shortest-path distances for all vertex pairs in $G$. For any  vertex pair $\{v, u\}$, the algorithm also records the last vertex before $u$ in a shortest path from $v$ to $u$ (we call this information \emph{backtracking tips}). With the backtracking tips, a shortest path between any two vertices can be obtained quickly (in time linear in the length of the path). While Procedure A-IT can be used many times, the Floyd-Warshall algorithm needs to be run only once, so it is viewed as pre-processing, and its complexity $O(n^3)$ is not counted towards the complexity of Procedure A-IT. Then, within an iteration of Procedure A-IT, to find a shortest path from $v$ to $T'$, we just need to check the distances $d(v,u)$, $u\in V'$, and for the vertex $u'\in V'$ that corresponds to the smallest such distance, to obtain the shortest path from $v$ to $u'$ using the backtracking tips. Hence, each iteration of the procedure takes $O(n)$ time, and there are no more than $n-1$ iterations. We obtain the following result.

\begin{thm}\label{complexity A-IT}
Given the inter-vertex shortest-path distances for $G$ and the backtracking tips, Procedure A-IT can be implemented in $O(n^2)$ time.
\end{thm}

 For IT problems, a solution defines a v-sequence according to the order of vertex recovery (the \emph{vertex recovery sequence}). Given  a v-sequence $S$, a \emph{vertex-shift move} corresponds to changing the position of one vertex in the sequence by moving it closer to the beginning of the sequence (by one or more positions). Let $N'(S)$ be the set of all v-sequences that can be obtained from $S$ by a single vertex-shift move. Observe that $|N'(S)|=O(n^2)$.

For a solution with a vertex recovery sequence $S$, its \emph{vertex-shift neighborhood} is defined as $\{ES(T'(S'))~|~S'\in N'(S)\}$, where $T'(S')$ is the spanning tree obtained by Procedure A-IT with input $S'$. That is, to obtain a vertex-shift neighbor of a solution, we apply a vertex-shift move to its vertex recovery sequence which results in a new v-sequence $S'$, then obtain the spanning tree $T'(S')$ using Procedure A-IT, and then obtain the optimal order of constructing the edges of $T'(S')$ using tree-efficiency of Problem NC. The vertex-shift neighborhood is scheduling-based since a neighbor is defined via a change in the vertex recovery sequence.

Consider now ET problems; as we discussed, we consider only ETPC problems.  A sequence of all $0.5n(n-1)$ vertex pairs of $G$ is called a \emph{p-sequence}. A solution defines the corresponding p-sequence according to the order of connection for the vertex pairs, which is called the \emph{pairs connection sequence} for the solution.   Some vertex pairs are connected simultaneously, but we assume that some unambiguous rule (e.g., lexicographic) is used for breaking such ties, so the pairs connection sequence for a solution is defined uniquely.  The pairs connection sequence for a solution consists of $n-1$ consecutive groups of vertex pairs, where the vertex pairs in each group are connected simultaneously as a result of constructing one essential edge; these groups will be referred to as \emph{p-groups}. 
In ET problems, constructing an essential edge results in joining two connected components of the previously constructed forest (the initial forest consists of $n$ isolated vertices), and this leads to connecting all pairs of vertices from these components, which defines the corresponding p-group.

For any p-sequence $S''$, consider the spanning tree that is obtained by the following procedure.

\textbf{Procedure A-ET}. \emph{Input:} a p-sequence $S''$.\\
\emph{Output:} spanning tree $T''(S'')$.\\
Set $G''=G$, $E''=\emptyset$, $T''=(V,E'')$. (At this point, $T''$ consists of $n$ isolated vertices. Throughout the course of the procedure, $T''$ will be a spanning forest, becoming a spanning tree for $G$ at the end. The vertices of $G''$ will represent the connected components of $T''$.)\\
Until $T''$ is a spanning tree for $G$ (i.e., until $|E''|=n-1$), do\\
(Main iteration)\\
\emph{Begin}\\
 $~~~~~$Select the first vertex pair from $S''$ that is still not connected by the edges in $T''$; let it be the pair $\{u, v\}$. Let $u''$, $v''$ be the vertices of $G''$ that correspond to $u$, $v$.\\
Find a shortest path between the vertices $u''$ and $v''$ in the current network $G''$ and include the edges of this path in $E''$. Update $T''=(V,E'')$. 
\\
  $~~~~~$Modify $G''$ by 
contracting the path into a single vertex. This may create loops and multi-edges. Delete all loops, and if there are several edges between two vertices, all but the shortest one are deleted (with some unambiguous rule for breaking ties).\\
\emph{End;}\\
Output $T''(S'')=T''$.

Procedure A-ET obtains a spanning tree $T''(S'')$ (clearly, no cycles are created).

\begin{thm}\label{complexity A-ET}  Procedure A-ET can be  implemented in $O(n^3)$ time.
\end{thm}
\begin{proof} 
Given the inter-vertex shortest-path distances and backtracking tips for network $G''$, a shortest path between any two vertices of $G''$ is obtained in time linear in the length of the path. Contracting the path takes time linear in the sum of the degrees of the vertices of the path, so contracting all paths in the course of the procedure  takes $O(n^2)$ time. After each iteration of the procedure, after contracting the path chosen in this iteration, we will obtain the shortest-path distances and backtracking tips for the updated network $G''$. This is clearly the bottleneck of complexity. Using the Floyd-Warshall algorithm for this purpose in each iteration would be inefficient since it takes $O(n^3)$ time, and as there are $O(n)$ iterations, this would result in $O(n^4)$ complexity for the procedure. However, the shortest-path distances and backtracking tips can be updated more directly and efficiently. Contracting a path can be represented as a sequential contraction of its edges. Suppose we have shortest-path distances $d(\cdot,\cdot)$ and backtracking tips for the current network $G''$, and we contract an edge $(x,y)$ in $G''$ into a single vertex $z$. Then, for any vertices $u$, $v$ of the modified network $G''$, the shortest-path distance $d'(u, v)$ is obtained in constant time as follows:
\[
d'(u, v)=
   \left\{
       \begin{array}{ll}
            \min\{ d(u,v), ~d(u,x)+d(y,v),~d(u,y)+d(x,v)\} & \mbox{if} ~u\ne z,~v\ne z,\\
            \min\{d(u, x),~d(u,y)\} & \mbox{if}~v=z,\\
            \min\{d(x,v),~d(y,v)\} & \mbox{if}~ u=z.
       \end{array}
    \right.
\]
Backtracking tips are updated similarly. So, updating all shortest-path distances and backtracking tips takes $O(n^2)$ time for each contracted edge, and, since there will be $n-1$ contracted edges, $O(n^3)$ time over the course of the procedure.
\end{proof}

For ETPC problems, a scheduling-based neighborhood is defined using vertex pairs. Given the pairs connection sequence $S$ for some solution, a \emph{pair-shift move} corresponds to changing the position of one vertex pair in the sequence by moving it closer to the beginning of the sequence, namely to the first position of some p-group that precedes the p-group of this vertex pair. If a vertex pair in position $j$ is moved to position $i$, $i<j$, then the vertex pairs that were in positions $i, i+1,...,j-1$ are moved one position up. Since there are $n-1$ p-groups, any vertex pair has no more than $n-2$ positions where it can be moved. Let $N''(S)$ be the set of all p-sequences that can be obtained from $S$ by a single pair-shift move. Observe that $|N''(S)|=O(n^3)$.

For a solution with the pairs connection sequence $S$, its \emph{pair-shift neighborhood} is defined as $\{ES(T''(S''))~|~S''\in N''(S)\}$, where $T''(S'')$ is the spanning tree obtained by Procedure A-ET with input $S''$. That is, to obtain a pair-shift neighbor of the current solution, we apply a pair-shift move to its pairs connection sequence which results in a new p-sequence $S''$, then obtain the spanning tree $T''(S'')$ using Procedure A-ET, and then obtain the optimal order of constructing the edges of $T''(S'')$ using tree-efficiency of Problem NC. The pair-shift neighborhood is scheduling-based since a neighbor is defined via a change in the p-sequence.


Using pair-shift neighborhoods for ETPC problems in the general case would clearly be time-consuming for larger instances (the cardinality of a pair-shift neighborhood is $O(n^3)$ and applying Procedure A-ET to a p-sequence takes $O(n^3)$ time, so a full search of a pair-shift neighborhood would take $O(n^6)$ time).  However, it is reasonable to expect that in many applications the number of \emph{relevant vertex pairs} (those that define the objective function) will not be large. Then, we can include only relevant vertex pairs in the pairs connection sequence, and consider only them for pair-shift moves that define neighbors, with straightforward modification of related procedures, which would reduce the complexity of the search.

\subsection{The solution improvement procedure}

The following observation clarifies the relation between Procedures A-IT and A-ET.

\begin{obs} Suppose that Problem NC is an IT problem, and consider any solution (e-sequence); let $S'$ be its vertex recovery sequence,  and let $S''$ be its pairs connection sequence. Then applying Procedure A-ET to $S''$ is equivalent to applying Procedure A-IT to $S'$ (i.e. produces the same spanning tree if no ties arise in the course of the procedures, or assuming that  tie-breaking rules in Procedures A-ET and A-IT are consistent).
\end{obs}
\begin{proof}
 Since we consider an e-sequence for an IT problem, for any $l\in [1:n-1]$, the first $l$ edges of the e-sequence form a connected subtree of $G$ that contains the depot $r$. Using induction, we will show that after any iteration $k$, the sets of edges $E'$ and $E''$ in Procedures A-IT and A-ET, respectively, are the same (assuming no ties arise). Initially, $E'=E''=\emptyset$. Suppose that after iteration $k-1$, $E'=E''$ (induction hypothesis). Let us prove that then $E'=E''$ after iteration $k$. After iteration $k-1$ of Procedures A-IT and A-ET, the set $E'=E''$ forms a subtree of $G$ that includes the depot $r$. If a vertex $v$ is selected by Procedure A-IT at iteration $k$, then a vertex pair that contains $v$ and some vertex already spanned by the edges from $E'=E''$ will be selected by Procedure A-ET at its iteration $k$. Then the edges of the same path (the shortest path between $v$ and the closest vertex to $v$ among the vertices of subtree $T'$ in $G$) will be added to $E'$ and $E''$ at iteration $k$. The observation follows immediately.
\end{proof}

Hence,  Procedure A-IT can be viewed as a simplified version of  Procedure A-ET adopted to IT problems.

\begin{lem}\label{Lemma1}
Consider any e-sequence $S$ for Problem NC.\\
a) If Problem NC is an IT problem and $S'$ is the vertex recovery sequence for $S$, then for the spanning tree $T'(S')$ obtained by applying Procedure A-IT to $S'$, the corresponding solution $ES(T'(S'))$ is not worse than $S$, i.e.,
\begin{equation}\label{Eq1}
F(ES(T'(S')))\le F(S).
\end{equation}
b) If Problem NC is an ET problem and $S''$ is the pairs connection sequence for $S$, then for the spanning tree $T''(S'')$ obtained by applying Procedure A-ET to $S''$, the corresponding solution $ES(T''(S''))$ is not worse than $S$, i.e.
\begin{equation}\label{Eq2}
F(ES(T''(S'')))\le F(S).
\end{equation}
\end{lem}
\begin{proof} a) Let $V'_k$, $E'_k$ be the sets $V'$, $E'$ from the description of Procedure A-IT after $k$-th iteration of Procedure A-IT applied to $S'$ (where $k$ is not larger than the number of iterations of Procedure A-IT), and let $v_k$ be the vertex $v$ selected at the $k$-th iteration of Procedure A-IT. Let $E_k$ be the set of edges that are constructed in the e-sequence $S$ before $v_k$ is recovered, and let $V_k$ be the set of already recovered vertices of $G$ after constructing the edges of $E_k$. In other words, $V_k$ is the set of vertices in the initial subsequence of the v-sequence $S'$ that ends with $v_k$. Using induction in $k$, we observe that for any $k$, we have $L(E'_k)\le L(E_k)$ and $V_k\subseteq V'_k$. Statement a) follows immediately.

b) Let $T''_k$ be the forest $T''$ from the description of Procedure A-ET after $k$-th iteration of Procedure A-ET applied to $S''$ (where $k$ is not larger than the number of iterations of Procedure A-ET), and let $\{u_k, v_k\}$ be the vertex pair $\{u,v\}$ selected at the $k$-th iteration of Procedure A-ET. Let $E_k$ be the set of edges that are constructed in the e-sequence $S$ before the vertex pair $\{u_k, v_k\}$  is connected.  For any vertex $v\in V$, let $V''_k(v)$ be the set of vertices of the connected component of $T''_k$ that contains $v$, and let $V_k(v)$ be the set of vertices of the connected component of the forest $(V, E_k)$ that contains $v$. Using induction in $k$, we observe that $L(T''_k)\le L(E_k)$, and  for any $v\in V$ we have  $V_k(v)\subseteq V''_k(v)$. Statement b) follows immediately.
\end{proof}
The inequalities in (\ref{Eq1}) and (\ref{Eq2}) may be strict even if $S$ is the optimal e-sequence $ES(T)$ for some spanning tree $T$, because a shortest path between two connected components of a forest may be shorter than the direct edge between two specific vertices from them, and/or may include other vertices.

Lemma \ref{Lemma1} motivates the following \emph{solution improvement procedure}. 

\textbf{Procedure IMPR}. \emph{Input:} A solution $S_0$, which is the optimal e-sequence for its spanning tree of essential edges $T$,  i.e. $S_0=ES(T)$.\\
\emph{Output:} An e-sequence $S_{\mbox{impr}}$.\\
1. Obtain the vertex recovery sequence $S'$ for $ES(T)$ if Problem NC is an IT problem (respectively, the pairs connection sequence $S''$ for $ES(T)$ if Problem NC is an ET problem).\\
2. Apply Procedure A-IT to $S'$ (respectively, Procedure A-ET to $S''$), obtaining a spanning tree $T'(S')$ (respectively, $T''(S''))$).\\
3. Obtain $ES(T'(S'))$ (respectively, $ES(T''(S''))$).\\
4. If $F(ES(T'(S')))<F(ES(T))$ (respectively, if $F(ES(T''(S'')))<F(ES(T))$), re-set $T=T'(S')$ (respectively, $T=T''(S'')$) and GO TO 1;\\
Otherwise, STOP and output $S_{\mbox{impr}}=ES(T)$.

\begin{thm}\label{impr}
  The solution $S_{\mbox{impr}}$ produced by Procedure IMPR is not worse than the starting solution $S_0$. If Procedure IMPR has more than one iteration, then $S_{\mbox{impr}}$ will be strictly better than $S_0$. 
\end{thm}
\begin{proof} The statement follows from Lemma \ref{Lemma1} and the description of Procedure IMPR.
\end{proof}
Observe that the complexity of one iteration of Procedure IMPR is $O(n^2+\tau(n))$ for IT problems, where $\tau(n)$ is the complexity of solving Problem NC on a tree, and $O(n^3+\tau(n))$ for ET problems.  For the tree-efficient IT problems considered in \cite{averbakhIIE, averbakhEJOR}, $\tau(n)=O(n\log n)$;  for the tree-efficient ET problems from \cite{averbakhJOC}, $\tau(n)=O(n\log n)$ with $O(n^3)$ pre-processing; hence, for these problems (which will be used for computational experiments), an iteration of Procedure IMPR takes $O(n^2)$ time for the IT problems from \cite{averbakhIIE, averbakhEJOR} and $O(n^3)$ time for the ET problems from \cite{averbakhJOC}.



\subsection{Generic local search}

Now we describe a local search procedure which incorporates Procedure IMPR, and the generic local-search based heuristic scheme  that was used in experiments.

\textbf{Procedure LOC}. \emph{Input:} A solution $A$, and the type of neighborhood to be used (chosen in advance).\\
1. Obtain the neighbors of $A$ one by one, and compute the corresponding objective values. As soon as an improving neighbor $A'$ is found (i.e., $F(A')<F(A)$), apply Procedure IMPR to $A'$, obtaining a solution $A''$; then, re-set $A=A''$, and return to 1.\\
2. If there is no improving neighbor, STOP and output $A$ as a local minimum.

\textbf{Heuristic MST-LOC}. Apply Heuristic MST, then apply Procedure LOC to the obtained solution.

The version of Procedure LOC that uses network-based (i.e., edge-exchange) neighborhoods will be referred to as \textbf{Procedure LOC-NET}; the version that uses scheduling-based neighborhoods (i.e., vertex-shift neighborhoods for IT problems and pair-shift neighborhoods for ET problems) will be referred to as \textbf{Procedure LOC-SCH}. Similarly, we define two versions of Heuristic MST-LOC, \textbf{Heuristic MST-LOC-NET} and \textbf{Heuristic MST-LOC-SCH}.

\section{Metaheuristics}\label{metaheuristics}

Two single-solution, local search based metaheuristics are considered, Tabu Search \cite{glover} and Iterated Local Search \cite{lourenco}. These metaheuristics operate with an \emph{incumbent} solution whose neighborhood  will be explored next, and finish when a time limit is reached reporting the \emph{best-so-far} solution. Each metaheuristic starts with running Heuristic MST-LOC (heuristic phase), and its local optimum solution becomes the initial solution for the metaheuristic phase. The metaheuristic phase and the heuristic phase use the same type of neighborhoods (network-based or scheduling-based). The spanning tree of the edges of the incumbent is called the \emph{incumbent spanning tree}.

\subsection{Tabu search} 
In TS, a tabu list is maintained that defines the local moves that are prohibited (tabu) unless they improve the best-so-far solution. A neighbor obtained by a local move prohibited by the tabu list is called a \emph{tabu neighbor}. Initially, the tabu list is empty, and the solution obtained by the heuristic phase is the incumbent and the best-so-far. At each iteration of TS, the whole neighborhood of the incumbent is explored (unlike the heuristic phase where a neighborhood is explored only until an improving solution is found), and the neighbor and the non-tabu neighbor with the best objective values are recorded. If the recorded neighbor has a better objective value than the best-so-far, the solution improving procedure (Procedure IMPR) is applied to this neighbor, the obtained solution becomes the new incumbent and the new best-so-far, and the tabu list is emptied. Otherwise, the best non-tabu neighbor becomes the new incumbent (even if it has a worse objective value than the current incumbent).

 When the incumbent changes without improving the best-so-far solution, each item characterizing the change is added to the tabu list for a number of iterations (tabu tenure) that is drawn randomly (for each item inclusion) from a uniform distribution between the \emph{minimum and maximum tabu tenure limits} which are the parameters of implementation. For edge-exchange neighborhoods, the items characterizing the change are the two edges that got exchanged, so the tabu list at any iteration is a set of edges; any edge exchange that involves an edge from the current tabu list is tabu. For vertex-shift neighborhoods, the item characterizing the change is the vertex moved ahead in the vertex recovery sequence, so the tabu list at every iteration is a set of vertices; any vertex-shift move initialized by a vertex in the current tabu list is tabu. For pair-shift neighborhoods, the items characterizing the change are the two vertices from the vertex pair that was moved ahead in the pairs connection sequence, so the tabu list at any iteration is a set of vertices; any pair-shift move initialized by a vertex pair that includes a vertex from the current tabu list is tabu. 

The choice of implementation parameters (the minimum and maximum tabu tenure limits) will be discussed later. TS with network-based (respectively, scheduling-based) neighborhoods will be referred to as \textbf{TS-NET} (respectively, \textbf{TS-SCH}).

\subsection{Iterated Local Search}
At each iteration of ILS, Procedure LOC is applied to the current incumbent. When a local optimum is reached, a \emph{shake operation} is performed, which represents randomly perturbing the incumbent solution; the degree of perturbation is defined  (probabilistically or deterministically) by a \emph{shake control parameter} $p$ which is a real number between 0 and 1 and is the parameter of implementation. Initially, the solution obtained by the heuristic phase is the incumbent and the best-so-far.
    
For problems with edge-exchange neighborhoods, the shake operation works as follows. Each edge of the incumbent is removed from the incumbent spanning tree independently with probability $p$. The remaining edges of the incumbent spanning tree define a forest. This forest is restored to a spanning tree by consecutively adding randomly chosen edges of $G$ that do not create cycles. The new spanning tree, after obtaining an optimal order of constructing its edges using tree-efficiency of Problem NC, defines the new incumbent. 

For problems with vertex-shift neighborhoods, the shake operation applies $\lceil p(n-1)\rceil$ random vertex-shift moves to the vertex recovery sequence of the incumbent. Then, Procedure A-IT is applied to the resulting v-sequence, and the obtained spanning tree, after computing an optimal order of constructing its edges using tree-efficiency of Problem NC, defines the new incumbent.

For problems with pair-shift neighborhoods, the shake operation applies $\lceil pq\rceil$ random pair-shift moves to the pairs connection sequence of the incumbent, where $q$ is the number of relevant vertex pairs (defined at the end of Subsection \ref{neighborhoods}). Only relevant vertex pairs can be considered for the pair-shift moves. After each pair-shift move, solution structure is restored using Procedure A-ET and tree-efficiency. The resulting solution is the new incumbent.

ILS with network-based (respectively, scheduling-based) neighborhoods will be referred to as \textbf{ILS-NET} (respectively, \textbf{ILS-SCH}).

\section{Computational experiments}\label{experiments}

\subsection{Specific NC problems}

For computational experiments, we consider the tree-efficient NC problems studied in \cite{averbakhIIE, averbakhEJOR, averbakhJOC}.

\textbf{Problem SWRT} (considered in \cite{averbakhIIE} and called Problem FNCP-W there; SWRT stands for ``sum of weighted recovery times"). This is the ITVR problem with the objective of minimizing the total weighted recovery time of all vertices $\sum_{v\in V\setminus\{r\}}w_vR(v)$, assuming that each vertex $v\ne r$ has an associated non-negative weight $w_v$. We will also consider the unweighted special case where all weights are equal to 1, which will be called \textbf{Problem USRT} (Problem FNCP in \cite{averbakhIIE}).

\textbf{Problem L} (considered in \cite{averbakhEJOR} and also called Problem L there). This is the ITVR problem with the objective of minimizing the maximum lateness of the vertices $\max_{v\in V\setminus \{r\}}(R(v)-d_v)$, assuming that each vertex $v\ne r$ has an associated due date $d_v$. Value $R(v)-d_v$ is called the \emph{lateness of vertex $v$}.

\textbf{Problem L-ETPC} (considered in \cite{averbakhJOC} and called Problem NCPC-L there). This is the ETPC problem where the objective is minimizing the maximum lateness of vertex pairs $\max_{\{u,v\}\in V^2,~v\ne u}(C(u,v)-d_{\{u,v\}})$, assuming that each vertex pair $\{u,v\}$ has an associated due date $d_{\{u,v\}}$. Value $C(u,v)-d_{\{u,v\}}$ is called the \emph{lateness of vertex pair $\{u, v\}$}.

All these problems are tree-efficient  \cite{averbakhIIE, averbakhEJOR, averbakhJOC}. Problems SWRT, USRT, and L on a tree can be solved in $O(n\log n)$ time  \cite{averbakhIIE, averbakhEJOR}, Problem L-ETPC on a tree can be solved in $O(n\log n+qn)$ time, where $q$ is the number of relevant vertex pairs (those that have finite due dates)  \cite{averbakhJOC}.

\subsection{Instances and implementation details}
For each problem, we used instances from the respective paper \cite{averbakhIIE, averbakhEJOR, averbakhJOC}. For full descriptions of the instance generating procedures, we refer the reader to \cite{averbakhIIE, averbakhEJOR, averbakhJOC}; here, we just briefly outline the nature of these procedures.

For Problems SWRT and USRT, the instances are complete networks that are divided into two groups, Euclidean instances and Random instances. For Euclidean instances, the vertices are randomly chosen in a square area of Euclidean plane, and the edge lengths are the rounded Euclidean distances between the endpoints. For Random instances, edge lengths are drawn randomly initially and then replaced with shortest-path distances for the resulting network (metric closure). For both groups, the depot location and vertex weights for weighted instances are generated randomly. (See \cite{averbakhIIE}.)

For Problem L, instances are sparse planar networks generated randomly to resemble road networks; each network has approximately 1.75$n$ edges, with vertices selected randomly from an Euclidean grid, and the lengths of the present edges corresponding to the Euclidean distances between the endpoints. The depot and the due dates for vertices are generated randomly. (See \cite{averbakhEJOR}.)

For Problem L-ETPC, instance networks are generated using the same procedure as for Problem L, and the due dates for vertex pairs are generated randomly ensuring that the number of relevant vertex pairs (those that have finite due dates) is 6$n$ (other vertex pairs are assumed to have due date $+\infty$). (See \cite{averbakhJOC}.)

To select values for the parameters of the metaheuristics (the shake control parameter for ILS and the maximum / minimum tabu tenure limits for TS), the Irace method \cite{IRACE} is used. Irace can be viewed as evolutionary optimization where individuals are parameter settings of a metaheuristic, and fitness of an individual is measured based on the performance of the metaheuristic on some instances of the problem. Different settings (individuals) are compared using iterated Friedman tests. A separate Irace run is made for each combination of a metaheuristic (TS or ILS), neighborhood type (network-based or scheduling-based), and objective function (4 of them), for a total of 16 combinations. The pool of instances used for running Irace for each combination is the set of largest instances used in the corresponding original paper (\cite{averbakhIIE}, \cite{averbakhEJOR}, or \cite{averbakhJOC}) for the corresponding objective. Specifically, these are Euclidean instances with $n=80$ and Random instances with $n=90$ from \cite{averbakhIIE} for Problem USRT, Euclidean instances with $n=45$ and Random instances with $n=40$ from \cite{averbakhIIE} for Problem SWRT, instances with $n=60$ from \cite{averbakhEJOR} for Problem L, and instances with $n=100$ for Problem L-ETPC. Table \ref{irace} provides the parameter values recommended by Irace and used in the experiments. ``ILS shake" corresponds to the ILS shake control parameter, ``TS tabu tenure" corresponds to minimum / maximum tabu tenure limits (the minimum limit is reported first, and the maximum limit is reported second), columns that correspond to versions with network-based or scheduling-based neighborhoods are identified respectively.

\begin{table}[htbp]
\begin{center}
\begin{tabular}{lcccc}
\hline
 & \multicolumn{2}{c}{ILS shake} & \multicolumn{2}{c}{TS tabu tenure}  \\
 & ILS-NET & ILS-SCH & TS-NET & TS-SCH \\
\hline 
USRT        & 0.11   & 0.36    & [5,17]  & [4,12] \\
SWRT   & 0.24  & 0.35  & [6,16] & [5,11] \\
L       & 0.23  & 0.46 & [7,14] & [7,17] \\
L-ETPC        & 0.03    & 0.14  & [7,14] & [5,17] \\
\hline
\end{tabular}
\caption{Values of control parameters recommended by Irace and used in the experiments.}
\label{irace}
\end{center}
\end{table}

The algorithms were coded in ANSI C++ and compiled with the GNU GCC compiler, version 8.3.1. The tests were run on a cluster consisting of 16 computers each having 32 Intel Xeon E5-2670 2.6GHz processors and 127 GB RAM running Linux. The code does not use any form of explicit parallelism and each processor is used to simultaneously solve a different instance (hence, we simultaneously solve 32 instances per computer). 

\subsection{Computational results}

Tables \ref{t_Completion} - \ref{t_pairwise} present the computational results for Problems USRT, SWRT, L, and L-ETPC, respectively. Column ``heuristic in paper" corresponds to the best-performing heuristic approach from the corresponding paper  \cite{averbakhIIE, averbakhEJOR, averbakhJOC}; ``exact" corresponds to the exact method (branch-and-bound) from the corresponding paper. The running time for the metaheuristics was limited to 600 seconds. For the exact methods, the results from the corresponding papers \cite{averbakhIIE, averbakhEJOR, averbakhJOC} were used; the heuristics from  \cite{averbakhIIE, averbakhEJOR, averbakhJOC} take very little time.  In Tables \ref{t_Completion} and \ref{t_wCompletion}, the instances are grouped by a combination of instance type (Euclidean or Random) and size $|V|$, 10 instances in each group. In Tables \ref{t_lateness} and \ref{t_pairwise}, the instances are grouped by the size $|V|$, 80 and 100 instances in each group, respectively. For each group - method combination, we report the number of instances from the group ``\#" for which the corresponding method found the best known solution (out of those found by all 6 methods), the average gap ``av." and the maximum gap ``max" for the solutions obtained by the method for the instances in the group. The gap for an instance - method combination is defined according to (\ref{gap1}) for Problems USRT and SWRT, and according to (\ref{gap2}) for Problems L and L-ETPC, where $UB$ is the objective value of the solution provided by the method for the instance, $best$ is the objective value of the best-known solution, and $d_{min}$ is the smallest due date in the instance (justification for including $d_{min}$ in (\ref{gap2}) was discussed in \cite{averbakhEJOR, averbakhJOC}).
\begin{equation}\label{gap1}
gap=100\times\frac{UB-best}{UB},
\end{equation}
\begin{equation}\label{gap2}
gap=100\times\frac{UB-best}{UB+d_{min}}.
\end{equation}
The results show the following:
\begin{itemize}
  \item Network-based neighborhoods are distinctly more effective for metaheuristics than scheduling-based neighborhoods.
  \item ILS is somewhat more effective than TS. It seems that TS is more prone to getting stuck in the neighborhood of a local optimum.
  \item The metaheuristics with edge-exchange neighborhoods are extremely effective. For almost all instances that were solved to optimality by exact methods in \cite{averbakhIIE, averbakhEJOR, averbakhJOC}, the  metaheuristics with edge-exchange neighborhoods found optimal solutions. ILS-NET found the best-known solutions for all instances except one (that corresponds to Problem L with $|V|=40$); a re-run of ILS-NET for this instance found the best-known solution. TS-NET found the best known solution for all instances except 16 (out of 2300); a re-run of TS-NET for these  instances with modified tabu tenure limits (the initial ones multiplied by 1.5) found the best-known solutions for all 16 instances. (However, the TS with the modified tabu tenure limits performed worse on several other instances, so the total number of instances where the best known solution was not found did not change significantly.)
  \item The metaheuristics with edge-exchange neighborhoods provide significant improvements over the heuristic methods used in \cite{averbakhIIE, averbakhEJOR, averbakhJOC}.
\end{itemize}

Hence, ILS-NET is the recommended metaheuristic for tree-efficient NC problems.

Tables \ref{overall-comparison-part1} and \ref{overall-comparison-part2} reflect a comparison between network-based and scheduling-based neighborhoods in the context of the local search heuristic without meta-heuristic components; that is, a comparison between Heuristics MST-LOC-NET and MST-LOC-SCH. In Table \ref{overall-comparison-part1}, each instance group is defined by a combination of problem (USRT or SWRT), instance type (Euclidean or Random), and instance size $|V|$; in Table \ref{overall-comparison-part2}, each instance group is defined by a combination of problem (L or L-ETPC) and instance size $|V|$. The column ``\# instances" reflects the number of instances in each group; the column ``\# best network" reports the number of instances in the group where Heuristic MST-LOC-NET obtained a strictly better solution than Heuristic MST-LOC-SCH; the column ``\# best scheduling" reports the number of instances in the group where Heuristic MST-LOC-SCH obtained a strictly better solution than Heuristic MST-LOC-NET. The results show that network-based neighborhoods are distinctly more effective than scheduling-based neighborhoods not only for metaheuristics but also for the local search without the metaheuristic phase, and the difference increases with the instance size (perhaps because for larger instances it becomes less likely to find the same solution). Hence, \emph{Heuristic MST-LOC-NET is the recommended heuristic approach} (if metaheuristics are not considered) for tree-efficient NC problems. 

Tables \ref{completion-comparison}, \ref{wcompletion-comparison}, \ref{lateness-comparison}, and \ref{pairwise-comparison} provide the results of comparison between Heuristic MST-LOC-NET and the best heuristic from the corresponding paper \cite{averbakhIIE, averbakhEJOR, averbakhJOC} for Problems USRT, SWRT, L, and L-ETPC. Columns ``\# better" reflect the number of instances where one heuristic found strictly better solutions than the other; also, the average and maximum values of the relevant gap ((\ref{gap1})or (\ref{gap2})) with respect to the best known solution are reported. The results demonstrate that Heuristic MST-LOC-NET performs distinctly better than the heuristics from \cite{averbakhIIE, averbakhEJOR} for Problems USRT, SWRT, and L. For Problem L-ETPC, Heuristic MST-LOC-NET performs somewhat better than the heuristic from \cite{averbakhJOC}, but the difference in gaps is very minor; this is expected because Heuristic MST-LOC-NET is quite similar to the heuristic from \cite{averbakhJOC} (the difference is the solution improvement procedure which is not used in \cite{averbakhJOC}). Tables \ref{completion-comparison} and \ref{wcompletion-comparison} also show that Heuristic MST-LOC-NET finds excellent solutions for Problems USRT and SWRT, so adding the metaheuristic phase for these problems has less importance than for Problems L and L-ETPC.

In additional computational experiments (not reported in the tables), we explored effectiveness of using the solution obtained by Heuristic MST as a starting solution for local search in metaheuristics (as described above) versus using a solution obtained from a randomly generated spanning tree, and effectiveness of using the solution improvement procedure (Procedure IMPR) versus not using it. The results showed that using both features improves the performance of the metaheuristics for all problems. Also, additional experiments indicated that the metaheuristics are robust with respect to moderate changes in the implementation parameters, so a possible practical approach could be to vary the implementation parameters in the course of the computation.

When the paper was in preparation, we learned about the recent work \cite{hermans}. In this paper, for a problem equivalent to Problem SWRT, an exact branch-and-cut algorithm and a polynomial approximation algorithm with a constant approximation ratio are developed. In addition, the authors consider a local search heuristic 
that uses edge-exchange neighborhoods, and observe that for the considered problem it has better performance than the local search heuristic from \cite{averbakhIIE} which uses scheduling-based neighborhoods. This is consistent with our conclusions that network-based neighborhoods are more effective than scheduling-based neighborhoods for tree-efficient NC problems.

\section{Conclusions}\label{conclusions}

In this paper, we proposed a generic heuristic approach for tree-efficient network construction problems based on local search and network-based (edge-exchange) neighborhoods (Heuristic MST-LOC-NET), which is a generalized and somewhat enhanced version of the local search heuristic used in \cite{averbakhJOC} for a specific network construction problem (Problem L-ETPC), and two metaheuristics (ILS-NET and TS-NET) which are the first metaheuristics for NC problems. The heuristic and metaheuristics are universal, in the sense that they are applicable to different types (external transportation, internal transportation) of NC problems with different objective functions, and use only the tree-efficiency property. Computational experiments were conducted for the specific tree-efficient NC problems studied in \cite{averbakhIIE, averbakhEJOR, averbakhJOC}. The experiments demonstrated that network-based neighborhoods are more effective than scheduling-based neighborhoods (which were used for local search heuristics in \cite{averbakhIIE, averbakhEJOR}) both in the context of heuristics and metaheuristics, and that metaheuristics ILS-NET and TS-NET are extremely effective. Specifically, for almost all instances solved to optimality by exact methods in \cite{averbakhIIE, averbakhEJOR, averbakhJOC}, the metaheuristics found optimal solutions.

There are countless possibilities to try to design different and more sophisticated versions of the metaheuristics, such as, for example, combining different types of neighborhoods (network-based and scheduling-based, single edge-exchange and double edge-exchange, etc.) in a Variable Neighborhood Search scheme, using metric closure of network $G$ instead of the network itself (for sparse networks this increases the size of network-based neighborhoods which may improve performance of local search, but would require more computational time per iteration), using multi-start approaches and different heuristics to obtain starting solutions, combining ILS and TS, etc. However, we chose simple ``basic" metaheuristic designs since, in our opinion, the main take-out of this research is the observation that for tree-efficient NC problems, metaheuristics and heuristics based on local search with network-based (edge-exchange) neighborhoods are extremely effective, even if designed generically using only tree-efficiency without exploiting specifics of particular  NC problems. We are not aware of  other network combinatorial optimization problems where tree-efficiency would be so useful and effective for heuristic / metaheuristic solution on general networks. Exploring similar ideas for other combinatorial optimization problems may be a direction for further research.

\vspace{0.3in}

\textbf{ACKNOWLEDGEMENT}. The research of Igor Averbakh was supported by the Discovery Grant RGPIN-2018-05066 from the Natural Sciences and Engineering Research Council of Canada (NSERC).

\begin{sidewaystable}[htbp]
\small
\begin{center}
\begin{tabular}{lc|ccc|ccc|ccc|ccc|ccc|ccc}
 &  & \multicolumn{3}{c|}{ILS-NET} & \multicolumn{3}{c|}{ILS-SCH} & \multicolumn{3}{c|}{TS-NET} & \multicolumn{3}{c|}{TS-SCH} & \multicolumn{3}{c|}{heuristic in paper} & \multicolumn{3}{c}{exact} \\
Type & $|V|$ & \# & av. & max. & \# & av. & max. & \# & av. & max. & \# & av. & max. & \# & av. & max. & \# & av. & max. \\
\hline
euclidean &  20  &  10  &  0.0  &  0.0  &  10  &  0.0  &  0.0  &  10  &  0.0  &  0.0  &  10  &  0.0  &  0.0  &  10  &  0.0  &  0.0  &  10  &  0.0  &  0.0  \\
euclidean &  25  &  10  &  0.0  &  0.0  &  10  &  0.0  &  0.0  &  10  &  0.0  &  0.0  &  10  &  0.0  &  0.0  &  10  &  0.0  &  0.0  &  10  &  0.0  &  0.0  \\
euclidean &  30  &  10  &  0.0  &  0.0  &  10  &  0.0  &  0.0  &  10  &  0.0  &  0.0  &  10  &  0.0  &  0.0  &  10  &  0.0  &  0.0  &  10  &  0.0  &  0.0  \\
euclidean &  35  &  10  &  0.0  &  0.0  &  10  &  0.0  &  0.0  &  10  &  0.0  &  0.0  &  10  &  0.0  &  0.0  &  10  &  0.0  &  0.0  &  10  &  0.0  &  0.0  \\
euclidean &  40  &  10  &  0.0  &  0.0  &  10  &  0.0  &  0.0  &  10  &  0.0  &  0.0  &  9  &  0.0  &  0.03  &  7  &  0.2  &  1.82  &  10  &  0.0  &  0.0  \\
euclidean &  45  &  10  &  0.0  &  0.0  &  9  &  0.02  &  0.17  &  10  &  0.0  &  0.0  &  8  &  0.07  &  0.55  &  5  &  0.38  &  1.57  &  10  &  0.0  &  0.0  \\
euclidean &  50  &  10  &  0.0  &  0.0  &  9  &  0.0  &  0.05  &  10  &  0.0  &  0.0  &  7  &  0.08  &  0.45  &  5  &  0.38  &  1.51  &  10  &  0.0  &  0.0  \\
euclidean &  55  &  10  &  0.0  &  0.0  &  9  &  0.04  &  0.41  &  10  &  0.0  &  0.0  &  4  &  0.27  &  0.98  &  3  &  0.3  &  0.98  &  10  &  0.0  &  0.0  \\
euclidean &  60  &  10  &  0.0  &  0.0  &  10  &  0.0  &  0.0  &  10  &  0.0  &  0.0  &  4  &  0.5  &  1.83  &  3  &  0.7  &  2.07  &  10  &  0.0  &  0.0  \\
euclidean &  65  &  10  &  0.0  &  0.0  &  8  &  0.03  &  0.22  &  10  &  0.0  &  0.0  &  4  &  0.2  &  0.54  &  2  &  0.46  &  2.55  &  8  &  0.02  &  0.21  \\
euclidean &  70  &  10  &  0.0  &  0.0  &  7  &  0.04  &  0.18  &  10  &  0.0  &  0.0  &  4  &  0.05  &  0.18  &  1  &  0.75  &  4.08  &  6  &  0.05  &  0.19  \\
euclidean &  80  &  10  &  0.0  &  0.0  &  5  &  0.08  &  0.36  &  10  &  0.0  &  0.0  &  2  &  0.4  &  1.11  &  1  &  0.49  &  1.54  &  5  &  0.08  &  0.37  \\
\hline
random &  20  &  10  &  0.0  &  0.0  &  10  &  0.0  &  0.0  &  10  &  0.0  &  0.0  &  10  &  0.0  &  0.0  &  10  &  0.0  &  0.0  &  10  &  0.0  &  0.0  \\
random &  25  &  10  &  0.0  &  0.0  &  10  &  0.0  &  0.0  &  10  &  0.0  &  0.0  &  10  &  0.0  &  0.0  &  10  &  0.0  &  0.0  &  10  &  0.0  &  0.0  \\
random &  30  &  10  &  0.0  &  0.0  &  10  &  0.0  &  0.0  &  10  &  0.0  &  0.0  &  10  &  0.0  &  0.0  &  9  &  0.06  &  0.55  &  10  &  0.0  &  0.0  \\
random &  35  &  10  &  0.0  &  0.0  &  9  &  0.02  &  0.17  &  10  &  0.0  &  0.0  &  7  &  0.35  &  2.21  &  6  &  0.65  &  4.87  &  10  &  0.0  &  0.0  \\
random &  40  &  10  &  0.0  &  0.0  &  10  &  0.0  &  0.0  &  10  &  0.0  &  0.0  &  10  &  0.0  &  0.0  &  10  &  0.0  &  0.0  &  10  &  0.0  &  0.0  \\
random &  45  &  10  &  0.0  &  0.0  &  10  &  0.0  &  0.0  &  10  &  0.0  &  0.0  &  10  &  0.0  &  0.0  &  9  &  0.0  &  0.01  &  10  &  0.0  &  0.0  \\
random &  50  &  10  &  0.0  &  0.0  &  9  &  0.29  &  2.91  &  10  &  0.0  &  0.0  &  8  &  0.31  &  2.91  &  7  &  0.31  &  2.91  &  10  &  0.0  &  0.0  \\
random &  55  &  10  &  0.0  &  0.0  &  9  &  0.02  &  0.17  &  10  &  0.0  &  0.0  &  8  &  0.07  &  0.54  &  7  &  0.18  &  1.17  &  10  &  0.0  &  0.0  \\
random &  60  &  10  &  0.0  &  0.0  &  9  &  0.18  &  1.83  &  10  &  0.0  &  0.0  &  5  &  0.25  &  1.83  &  5  &  0.31  &  1.83  &  10  &  0.0  &  0.0  \\
random &  65  &  10  &  0.0  &  0.0  &  10  &  0.0  &  0.0  &  10  &  0.0  &  0.0  &  7  &  0.06  &  0.31  &  6  &  0.18  &  0.84  &  10  &  0.0  &  0.0  \\
random &  70  &  10  &  0.0  &  0.0  &  9  &  0.03  &  0.29  &  10  &  0.0  &  0.0  &  8  &  0.15  &  0.94  &  7  &  0.23  &  1.69  &  10  &  0.0  &  0.0  \\
random &  80  &  10  &  0.0  &  0.0  &  5  &  0.08  &  0.25  &  9  &  0.0  &  0.01  &  2  &  0.39  &  1.13  &  2  &  0.98  &  2.51  &  4  &  0.14  &  0.35  \\
random &  90  &  10  &  0.0  &  0.0  &  8  &  0.01  &  0.04  &  10  &  0.0  &  0.0  &  4  &  0.13  &  0.86  &  3  &  0.33  &  1.1  &  3  &  0.12  &  0.45  \\
\hline
\end{tabular}
\caption{Results for Problem USRT.}
\label{t_Completion}
\end{center}
\end{sidewaystable}

\begin{sidewaystable}[htbp]
\small
\begin{center}
\begin{tabular}{lc|ccc|ccc|ccc|ccc|ccc|ccc}
 &  & \multicolumn{3}{c|}{ILS-NET} & \multicolumn{3}{c|}{ILS-SCH} & \multicolumn{3}{c|}{TS-NET} & \multicolumn{3}{c|}{TS-SCH} & \multicolumn{3}{c|}{heuristic in paper} & \multicolumn{3}{c}{exact} \\
Type & $|V|$ & \# & av. & max. & \# & av. & max. & \# & av. & max. & \# & av. & max. & \# & av. & max. & \# & av. & max. \\
\hline
euclidean &  20  &  10  &  0.0  &  0.0  &  10  &  0.0  &  0.0  &  10  &  0.0  &  0.0  &  10  &  0.0  &  0.0  &  10  &  0.0  &  0.0  &  10  &  0.0  &  0.0  \\
euclidean &  25  &  10  &  0.0  &  0.0  &  10  &  0.0  &  0.0  &  10  &  0.0  &  0.0  &  9  &  0.0  &  0.02  &  7  &  0.53  &  5.29  &  10  &  0.0  &  0.0  \\
euclidean &  30  &  10  &  0.0  &  0.0  &  10  &  0.0  &  0.0  &  10  &  0.0  &  0.0  &  9  &  0.0  &  0.0  &  7  &  0.07  &  0.74  &  10  &  0.0  &  0.0  \\
euclidean &  35  &  10  &  0.0  &  0.0  &  10  &  0.0  &  0.0  &  10  &  0.0  &  0.0  &  8  &  0.03  &  0.27  &  5  &  0.25  &  1.55  &  10  &  0.0  &  0.0  \\
euclidean &  40  &  10  &  0.0  &  0.0  &  9  &  0.09  &  0.92  &  9  &  0.11  &  1.07  &  6  &  0.25  &  1.07  &  4  &  1.0  &  7.37  &  7  &  0.16  &  0.96  \\
euclidean &  45  &  10  &  0.0  &  0.0  &  9  &  0.0  &  0.04  &  10  &  0.0  &  0.0  &  6  &  0.53  &  2.72  &  5  &  1.17  &  5.59  &  6  &  0.4  &  2.65  \\
\hline
random &  20  &  10  &  0.0  &  0.0  &  10  &  0.0  &  0.0  &  10  &  0.0  &  0.0  &  10  &  0.0  &  0.0  &  10  &  0.0  &  0.0  &  10  &  0.0  &  0.0  \\
random &  25  &  10  &  0.0  &  0.0  &  10  &  0.0  &  0.0  &  10  &  0.0  &  0.0  &  10  &  0.0  &  0.0  &  10  &  0.0  &  0.0  &  10  &  0.0  &  0.0  \\
random &  30  &  10  &  0.0  &  0.0  &  10  &  0.0  &  0.0  &  10  &  0.0  &  0.0  &  9  &  0.0  &  0.0  &  8  &  0.07  &  0.74  &  10  &  0.0  &  0.0  \\
random &  35  &  10  &  0.0  &  0.0  &  10  &  0.0  &  0.0  &  10  &  0.0  &  0.0  &  8  &  0.02  &  0.2  &  7  &  0.81  &  7.88  &  10  &  0.0  &  0.0  \\
random &  40  &  10  &  0.0  &  0.0  &  10  &  0.0  &  0.0  &  10  &  0.0  &  0.0  &  10  &  0.0  &  0.0  &  9  &  0.03  &  0.29  &  9  &  0.03  &  0.29  \\
\hline
\end{tabular}
\caption{Results for Problem SWRT.}
\label{t_wCompletion}
\end{center}
\end{sidewaystable}

\begin{sidewaystable}[htbp]
\small
\begin{center}
\begin{tabular}{l|ccc|ccc|ccc|ccc|ccc|ccc}
   & \multicolumn{3}{c|}{ILS-NET} & \multicolumn{3}{c|}{ILS-SCH} & \multicolumn{3}{c|}{TS-NET} & \multicolumn{3}{c|}{TS-SCH} & \multicolumn{3}{c|}{heuristic in paper} & \multicolumn{3}{c}{exact} \\
$|V|$ & \# & av. & max. & \# & av. & max. & \# & av. & max. & \# & av. & max. & \# & av. & max. & \# & av. & max. \\
\hline
25  &  80  &  0.0  &  0.0  &  76  &  0.08  &  4.41  &  80  &  0.0  &  0.0  &  68  &  0.3  &  4.51  &  48  &  1.38  &  15.44  &  80  &  0.0  &  0.0  \\
30  &  80  &  0.0  &  0.0  &  74  &  0.04  &  1.71  &  80  &  0.0  &  0.0  &  66  &  0.35  &  5.25  &  49  &  1.4  &  21.56  &  80  &  0.0  &  0.0  \\
35  &  80  &  0.0  &  0.0  &  76  &  0.02  &  0.68  &  80  &  0.0  &  0.0  &  63  &  0.37  &  4.61  &  46  &  1.59  &  14.77  &  80  &  0.0  &  0.0  \\
40  &  79  &  0.0  &  0.12  &  70  &  0.13  &  2.91  &  80  &  0.0  &  0.0  &  61  &  0.65  &  13.37  &  39  &  1.66  &  25.87  &  80  &  0.0  &  0.0  \\
45  &  80  &  0.0  &  0.0  &  69  &  0.09  &  3.22  &  80  &  0.0  &  0.0  &  55  &  0.49  &  7.62  &  34  &  1.68  &  11.19  &  78  &  0.02  &  0.83  \\
50  &  80  &  0.0  &  0.0  &  64  &  0.12  &  1.3  &  80  &  0.0  &  0.0  &  54  &  0.46  &  6.68  &  37  &  1.64  &  12.5  &  69  &  0.36  &  7.06  \\
55  &  80  &  0.0  &  0.0  &  58  &  0.25  &  4.03  &  77  &  0.01  &  0.47  &  53  &  0.41  &  4.8  &  28  &  1.81  &  12.1  &  59  &  0.39  &  10.07  \\
60  &  80  &  0.0  &  0.0  &  61  &  0.25  &  4.3  &  78  &  0.02  &  1.36  &  46  &  0.7  &  10.99  &  30  &  1.85  &  16.24  &  46  &  1.28  &  16.24  \\
\hline
\end{tabular}
\caption{Results for Problem L.}
\label{t_lateness}
\end{center}
\end{sidewaystable}

\begin{sidewaystable}[htbp]
\small
\begin{center}
\begin{tabular}{l|ccc|ccc|ccc|ccc|ccc|ccc}
   & \multicolumn{3}{c|}{ILS-NET} & \multicolumn{3}{c|}{ILS-SCH} & \multicolumn{3}{c|}{TS-NET} & \multicolumn{3}{c|}{TS-SCH} & \multicolumn{3}{c|}{heuristic in paper} & \multicolumn{3}{c}{exact} \\
$|V|$ & \# & av. & max. & \# & av. & max. & \# & av. & max. & \# & av. & max. & \# & av. & max. & \# & av. & max. \\
\hline
20  &  100  &  0.0  &  0.0  &  96  &  0.04  &  2.55  &  100  &  0.0  &  0.0  &  47  &  0.67  &  5.17  &  92  &  0.05  &  1.48  &  100  &  0.0  &  0.0  \\
25  &  100  &  0.0  &  0.0  &  81  &  0.13  &  2.67  &  100  &  0.0  &  0.0  &  47  &  0.52  &  5.61  &  89  &  0.08  &  1.6  &  100  &  0.0  &  0.0  \\
30  &  100  &  0.0  &  0.0  &  76  &  0.17  &  2.79  &  100  &  0.0  &  0.0  &  43  &  0.55  &  3.92  &  85  &  0.14  &  2.05  &  100  &  0.0  &  0.0  \\
35  &  100  &  0.0  &  0.0  &  72  &  0.25  &  4.93  &  100  &  0.0  &  0.0  &  52  &  0.53  &  6.38  &  92  &  0.09  &  2.59  &  100  &  0.0  &  0.0  \\
40  &  100  &  0.0  &  0.0  &  68  &  0.33  &  4.37  &  100  &  0.0  &  0.0  &  45  &  0.49  &  4.37  &  85  &  0.08  &  1.64  &  100  &  0.0  &  0.0  \\
45  &  100  &  0.0  &  0.0  &  66  &  0.27  &  2.75  &  100  &  0.0  &  0.0  &  52  &  0.45  &  4.86  &  88  &  0.04  &  1.3  &  98  &  0.01  &  0.35  \\
50  &  100  &  0.0  &  0.0  &  59  &  0.32  &  2.37  &  100  &  0.0  &  0.0  &  50  &  0.38  &  2.7  &  77  &  0.12  &  1.98  &  96  &  0.01  &  0.27  \\
55  &  100  &  0.0  &  0.0  &  49  &  0.34  &  4.77  &  100  &  0.0  &  0.0  &  33  &  0.46  &  4.77  &  76  &  0.1  &  1.62  &  98  &  0.0  &  0.35  \\
60  &  100  &  0.0  &  0.0  &  54  &  0.33  &  2.35  &  99  &  0.0  &  0.02  &  45  &  0.39  &  2.58  &  81  &  0.08  &  1.92  &  98  &  0.0  &  0.06  \\
70  &  100  &  0.0  &  0.0  &  44  &  0.34  &  3.78  &  100  &  0.0  &  0.0  &  34  &  0.37  &  3.78  &  78  &  0.05  &  0.97  &  89  &  0.02  &  0.77  \\
80  &  100  &  0.0  &  0.0  &  44  &  0.26  &  1.35  &  100  &  0.0  &  0.0  &  42  &  0.26  &  1.35  &  83  &  0.09  &  1.93  &  88  &  0.04  &  1.06  \\
90  &  100  &  0.0  &  0.0  &  38  &  0.39  &  2.53  &  93  &  0.02  &  0.61  &  33  &  0.41  &  3.56  &  70  &  0.19  &  1.97  &  83  &  0.11  &  1.97  \\
100  &  100  &  0.0  &  0.0  &  40  &  0.37  &  2.51  &  98  &  0.01  &  0.83  &  35  &  0.39  &  2.51  &  72  &  0.09  &  1.61  &  82  &  0.05  &  0.83  \\
\hline
\end{tabular}
\caption{Results for Problem L-ETPC.}
\label{t_pairwise}
\end{center}
\end{sidewaystable}

\begin{table}[htbp]
\small
\begin{center}
\begin{tabular}{lcccc}
\hline
type & $|V|$ & \# instances & \# best network & \# best scheduling \\
\hline
USRT (Euclidean) &  20  &  10  &  3  &  0 \\
USRT (Euclidean) &  25  &  10  &  2  &  0 \\
USRT (Euclidean) &  30  &  10  &  2  &  0 \\
USRT (Euclidean) &  35  &  10  &  0  &  0 \\
USRT (Euclidean) &  40  &  10  &  3  &  0 \\
USRT (Euclidean) &  45  &  10  &  5  &  0 \\
USRT (Euclidean) &  50  &  10  &  5  &  0 \\
USRT (Euclidean) &  55  &  10  &  6  &  0 \\
USRT (Euclidean) &  60  &  10  &  6  &  0 \\
USRT (Euclidean) &  65  &  10  &  7  &  1 \\
USRT (Euclidean) &  70  &  10  &  6  &  0 \\
USRT (Euclidean) &  80  &  10  &  8  &  0 \\
\hline
USRT (Random) &  20  &  10  &  0  &  0 \\
USRT (Random) &  25  &  10  &  0  &  0 \\
USRT (Random) &  30  &  10  &  1  &  0 \\
USRT (Random) &  35  &  10  &  3  &  0 \\
USRT (Random) &  40  &  10  &  1  &  0 \\
USRT (Random) &  45  &  10  &  0  &  0 \\
USRT (Random) &  50  &  10  &  3  &  0 \\
USRT (Random) &  55  &  10  &  2  &  0 \\
USRT (Random) &  60  &  10  &  5  &  0 \\
USRT (Random) &  65  &  10  &  3  &  0 \\
USRT (Random) &  70  &  10  &  3  &  0 \\
USRT (Random) &  80  &  10  &  6  &  1 \\
USRT (Random) &  90  &  10  &  7  &  0 \\
\hline
SWRT (Euclidean) &  20  &  10  &  1  &  0 \\
SWRT (Euclidean) &  25  &  10  &  4  &  0 \\
SWRT (Euclidean) &  30  &  10  &  4  &  0 \\
SWRT (Euclidean) &  35  &  10  &  4  &  0 \\
SWRT (Euclidean) &  40  &  10  &  2  &  1 \\
SWRT (Euclidean) &  45  &  10  &  6  &  0 \\
\hline
SWRT (Random) &  20  &  10  &  0  &  0 \\
SWRT (Random) &  25  &  10  &  3  &  0 \\
SWRT (Random) &  30  &  10  &  3  &  0 \\
SWRT (Random) &  35  &  10  &  5  &  0 \\
SWRT (Random) &  40  &  10  &  1  &  0 \\
\end{tabular}
\caption{Comparison between network and scheduling neighborhoods for Heuristic MST-LOC for Problems USRT and SWRT.} 
\label{overall-comparison-part1}
\end{center}
\end{table}

\begin{table}[htbp]
\small
\begin{center}
\begin{tabular}{lcccc}
\hline
type & $|V|$ & \# instances & \# best network & \# best scheduling \\
\hline
L &  25  &  80  &  21  &  5 \\
L &  30  &  80  &  27  &  3 \\
L &  35  &  80  &  26  &  3 \\
L &  40  &  80  &  25  &  9 \\
L &  45  &  80  &  35  &  5 \\
L &  50  &  80  &  31  &  8 \\
L &  55  &  80  &  35  &  7 \\
L &  60  &  80  &  36  &  7 \\
\hline
L-ETPC &  20  &  100  &  17  &  0 \\
L-ETPC &  25  &  100  &  19  &  0 \\
L-ETPC &  30  &  100  &  25  &  4 \\
L-ETPC &  35  &  100  &  26  &  2 \\
L-ETPC &  40  &  100  &  30  &  0 \\
L-ETPC &  45  &  100  &  29  &  3 \\
L-ETPC &  50  &  100  &  28  &  3 \\
L-ETPC &  55  &  100  &  43  &  2 \\
L-ETPC &  60  &  100  &  41  &  2 \\
L-ETPC &  70  &  100  &  44  &  3 \\
L-ETPC &  80  &  100  &  47  &  3 \\
L-ETPC &  90  &  100  &  52  &  2 \\
L-ETPC &  100  &  100  &  53  &  1 \\
\hline
\end{tabular}
\caption{Comparison between network and scheduling neighborhoods for Heuristic MST-LOC for Problems L and L-ETPC.} 
\label{overall-comparison-part2}
\end{center}
\end{table}

\begin{table}[htbp]
\begin{center}
\begin{tabular}{llcccccc}
\hline
 &  & \multicolumn{3}{c}{Heuristic MST-LOC-NET} & \multicolumn{3}{c}{Best heuristic in paper}  \\
type & $|V|$ & \# better & av. gap & max. gap & \# better & av. gap & max. gap \\
\hline
Euclidean &  20  &  0  &  0.0  &  0.0  &  0  &  0.0  &  0.0  \\
Euclidean &  25  &  0  &  0.0  &  0.0  &  0  &  0.0  &  0.0  \\
Euclidean &  30  &  0  &  0.0  &  0.0  &  0  &  0.0  &  0.0  \\
Euclidean &  35  &  0  &  0.0  &  0.0  &  0  &  0.0  &  0.0  \\
Euclidean &  40  &  3  &  0.0  &  0.0  &  0  &  0.2  &  1.82  \\
Euclidean &  45  &  5  &  0.0  &  0.0  &  0  &  0.38  &  1.57  \\
Euclidean &  50  &  5  &  0.0  &  0.0  &  0  &  0.38  &  1.51  \\
Euclidean &  55  &  7  &  0.0  &  0.0  &  0  &  0.3  &  0.98  \\
Euclidean &  60  &  7  &  0.02  &  0.16  &  0  &  0.7  &  2.07  \\
Euclidean &  65  &  8  &  0.0  &  0.05  &  0  &  0.46  &  2.55  \\
Euclidean &  70  &  9  &  0.02  &  0.19  &  0  &  0.75  &  4.08  \\
Euclidean &  80  &  9  &  0.0  &  0.02  &  0  &  0.49  &  1.54  \\
\hline
Random &  20  &  0  &  0.0  &  0.0  &  0  &  0.0  &  0.0  \\
Random &  25  &  0  &  0.0  &  0.0  &  0  &  0.0  &  0.0  \\
Random &  30  &  1  &  0.0  &  0.0  &  0  &  0.06  &  0.55  \\
Random &  35  &  4  &  0.0  &  0.0  &  0  &  0.65  &  4.87  \\
Random &  40  &  0  &  0.0  &  0.0  &  0  &  0.0  &  0.0  \\
Random &  45  &  1  &  0.0  &  0.0  &  0  &  0.0  &  0.01  \\
Random &  50  &  3  &  0.0  &  0.0  &  0  &  0.31  &  2.91  \\
Random &  55  &  3  &  0.0  &  0.0  &  0  &  0.18  &  1.17  \\
Random &  60  &  5  &  0.0  &  0.0  &  0  &  0.31  &  1.83  \\
Random &  65  &  4  &  0.0  &  0.0  &  0  &  0.18  &  0.84  \\
Random &  70  &  3  &  0.0  &  0.0  &  0  &  0.23  &  1.69  \\
Random &  80  &  7  &  0.15  &  1.04  &  0  &  0.98  &  2.51  \\
Random &  90  &  7  &  0.0  &  0.0  &  0  &  0.33  &  1.1  \\
\hline
\end{tabular}
\caption{Comparison between the best heuristic reported in \cite{averbakhIIE} and Heuristic MST-LOC-NET for Problem USRT. 10 instances per a type-size group.}
\label{completion-comparison}
\end{center}
\end{table}

\begin{table}[htbp]
\begin{center}
\begin{tabular}{llcccccc}
\hline
 &  & \multicolumn{3}{c}{Heuristic MST-LOC-NET} & \multicolumn{3}{c}{Best heuristic in paper}  \\
type & $|V|$ & \# better & av. gap & max. gap & \# better & av. gap & max. gap \\
\hline
Euclidean &  20  &  0  &  0.0  &  0.0  &  0  &  0.0  &  0.0  \\
Euclidean &  25  &  3  &  0.0  &  0.0  &  0  &  0.53  &  5.29  \\
Euclidean &  30  &  3  &  0.0  &  0.0  &  0  &  0.07  &  0.74  \\
Euclidean &  35  &  5  &  0.0  &  0.0  &  0  &  0.25  &  1.55  \\
Euclidean &  40  &  3  &  0.25  &  1.07  &  0  &  1.0  &  7.37  \\
Euclidean &  45  &  5  &  0.0  &  0.04  &  0  &  1.17  &  5.59  \\
\hline
Random &  20  &  0  &  0.0  &  0.0  &  0  &  0.0  &  0.0  \\
Random &  25  &  0  &  0.0  &  0.0  &  0  &  0.0  &  0.0  \\
Random &  30  &  2  &  0.0  &  0.0  &  0  &  0.07  &  0.74  \\
Random &  35  &  3  &  0.0  &  0.0  &  0  &  0.81  &  7.88  \\
Random &  40  &  1  &  0.0  &  0.0  &  0  &  0.03  &  0.29  \\
\hline
\end{tabular}
\caption{Comparison between the best heuristic reported in \cite{averbakhIIE} and Heuristic MST-LOC-NET for Problem SWRT. 10 instances per each type-size group.}
\label{wcompletion-comparison}
\end{center}
\end{table}

\begin{table}[htbp]
\begin{center}
\begin{tabular}{lcccccc}
\hline
  & \multicolumn{3}{c}{Heuristic MST-LOC-NET} & \multicolumn{3}{c}{Best heuristic in paper}  \\
 $|V|$ & \# better & av. gap & max. gap & \# better & av. gap & max. gap \\
\hline
25  &  25  &  0.29  &  3.59  &  2  &  1.38  &  15.44 \\
30  &  24  &  0.29  &  4.05  &  4  &  1.4  &  21.56 \\
35  &  31  &  0.14  &  3.38  &  1  &  1.59  &  14.77 \\
40  &  32  &  0.57  &  7.24  &  5  &  1.66  &  25.87 \\
45  &  37  &  0.31  &  3.46  &  6  &  1.68  &  11.19 \\
50  &  34  &  0.43  &  5.0  &  9  &  1.64  &  12.5 \\
55  &  43  &  0.51  &  6.82  &  4  &  1.81  &  12.1 \\
60  &  42  &  0.44  &  7.2  &  3  &  1.85  &  16.24 \\
\hline
\end{tabular}
\caption{Comparison between the best heuristic reported in \cite{averbakhEJOR} and Heuristic MST-LOC-NET for Problem L. 80 instances per each size group.}
\label{lateness-comparison}
\end{center}
\end{table}

\begin{table}[htbp]
\begin{center}
\begin{tabular}{lcccccc}
\hline
  & \multicolumn{3}{c}{Heuristic MST-LOC-NET} & \multicolumn{3}{c}{Best heuristic in paper}  \\
 $|V|$ & \# better & av. gap & max. gap & \# better & av. gap & max. gap \\
\hline
20  &  4  &  0.02  &  1.35  &  0  &  0.05  &  1.48 \\
25  &  4  &  0.04  &  1.6  &  0  &  0.08  &  1.6 \\
30  &  2  &  0.14  &  2.05  &  1  &  0.14  &  2.05 \\
35  &  2  &  0.1  &  2.59  &  1  &  0.09  &  2.59 \\
40  &  4  &  0.06  &  1.64  &  0  &  0.08  &  1.64 \\
45  &  2  &  0.05  &  1.3  &  2  &  0.04  &  1.3 \\
50  &  5  &  0.1  &  1.98  &  1  &  0.12  &  1.98 \\
55  &  7  &  0.05  &  1.0  &  0  &  0.1  &  1.62 \\
60  &  8  &  0.04  &  1.43  &  1  &  0.08  &  1.92 \\
70  &  6  &  0.05  &  0.99  &  5  &  0.05  &  0.97 \\
80  &  8  &  0.05  &  1.46  &  1  &  0.09  &  1.93 \\
90  &  16  &  0.08  &  1.06  &  2  &  0.19  &  1.97 \\
100  &  14  &  0.05  &  1.11  &  1  &  0.09  &  1.61 \\
\hline
\end{tabular}
\caption{Comparison between the best heuristic reported in \cite{averbakhJOC} and Heuristic MST-LOC-NET for Problem L-ETPC. 100 instances per each size group.}
\label{pairwise-comparison}
\end{center}
\end{table}

\end{document}